\documentclass{article}

\usepackage[preprint]{neurips_2026}

\usepackage[utf8]{inputenc} 
\usepackage[T1]{fontenc}    
\usepackage{microtype}
\usepackage{graphicx}
\usepackage{subcaption}
\usepackage{booktabs} 
\usepackage{array}
\usepackage{tabularx}
\usepackage{float}          
\usepackage{nicefrac}       
\usepackage{xcolor}         
\usepackage{url}            
\newcolumntype{L}{>{\raggedright\arraybackslash}X}

\usepackage{hyperref}

\usepackage{amsmath}
\usepackage{amssymb}
\usepackage{amsfonts}
\usepackage{mathtools}
\usepackage{amsthm}

\usepackage[capitalize,noabbrev]{cleveref}

\theoremstyle{plain}
\newtheorem{theorem}{Theorem}[section]
\newtheorem{proposition}[theorem]{Proposition}

\theoremstyle{definition}

\theoremstyle{remark}

\usepackage[disable,textsize=tiny]{todonotes}

\title{Gated Normalization Removal and Scale Anchoring in Pre-Norm Transformers}

\author{%
  Andrei Kanavalau\thanks{Corresponding author: \texttt{kanaval@stanford.edu}.} \\
  Department of Electrical Engineering\\
  Stanford University\\
  Stanford, USA \\
  \And
  Carmen Amo Alonso \\
  Department of Computer Science\\
  Stanford University\\
  Stanford, USA \\
  \And
  Sanjay Lall \\
  Department of Electrical Engineering\\
  Stanford University\\
  Stanford, USA \\
}

\begin{document}

\maketitle

\begin{abstract}
  Normalization layers are standard in transformers, but it is not clear whether their sample-dependent computations are necessary throughout both training and inference. This work develops a gated normalization-removal approach for pre-norm transformers. The approach is implemented using TaperNorm, which starts from standard RMSNorm/LayerNorm and gradually tapers to learned sample-independent linear or affine maps. Once the gate reaches zero, per-token statistics are no longer computed in the tapered layers and the resulting maps can be folded into adjacent linear projections. The results indicate that internal normalization can be tapered in the tested pre-training and fine-tuning settings with small validation-loss increases. Our approach helps reveal a distinct role for final normalization, namely that it anchors the scale of the pre-logit representation. With this anchor present, radial changes in the last hidden state do not directly reduce the loss; when it is removed, reducing cross-entropy can be achieved by increasing logit magnitudes. A fixed-target scale loss provides an explicit alternative anchor and enables fully norm-free ablations in the tested regimes. Finally, in a KV-cached autoregressive decoding benchmark, tapering internal norms gives up to $1.14\times$ higher throughput with explicit scaling operations and up to $1.18\times$ after folding.
\end{abstract}

\section{Introduction}

Normalization layers combined with residual connections were instrumental in making very deep networks trainable. Residual paths preserve gradients \citep{he2015deepresiduallearningimage} while normalization layers such as BatchNorm, LayerNorm, and RMSNorm stabilize activations \citep{ioffe2015batchnormalizationacceleratingdeep,ba2016layernormalization,zhang2019rootmeansquarelayer}. LayerNorm and RMSNorm are commonly used in transformer based language models where normalization is applied per-token. Such normalization introduces a nonlinear data-dependent operation into every sublayer and incurs computational cost. This work investigates whether internal per-token normalization in pre-norm transformers must remain active after training.

To this end we introduce a gated normalization-removal approach. The corresponding module, TaperNorm, begins as ordinary RMSNorm/LayerNorm and then tapers to a sample-independent linear or affine map. A single global gate controls the transition. Early in training the model behaves exactly like a pre-norm transformer. As the gate decays, selected normalizers are smoothly tapered toward learned sample-independent scaling. At convergence, per-token statistics are no longer used in the tapered layers and their effect can be absorbed into adjacent linear projections.

This work uses theory and empirical observations to identify the distinct role played by the final normalization layer that sets it apart compared to other normalization layers inside the model. The analysis shows that final normalization controls the scale of the pre-logit representation. A final normalization layer removes radial gradients at the output, so the loss cannot be reduced by simply increasing the last hidden-state norm. Without such an anchor, cross-entropy has a radial component that increases logit magnitudes. This makes removal of the final normalization layer challenging. We find that a fixed-target auxiliary loss on the pre-logit residual stream introduces a counteracting radial restoring force that allows us to study the behavior in more detail.

In our default setting, Internal-Taper, the internal block normalizers are tapered while the final normalization remains active. This removes the repeated normalizers inside transformer blocks while preserving the final normalization as an implicit output scale anchor. In the All-Taper setting, the final normalization is tapered as well. This produces a fully norm-free model at inference, but is more fragile and is used mainly to investigate the role of output scale anchoring.

Our key contributions are:
\begin{itemize}
    \item A gated normalization-removal approach that tapers selected RMSNorm/LayerNorm modules to fixed linear or affine rescalings, allowing per-token statistics to be removed at inference.
    \item Analysis of output scale anchoring in pre-norm transformers, showing that final normalization removes radial gradients, while minimizing cross-entropy induces logit growth if radial gradients are left unchecked.
    \item Experiments on TinyStories and GPT-2 characterizing internal normalization removal, with small loss increases in the tested regimes and up to $1.18\times$ higher throughput after folding internal scalings.
\end{itemize}

\section{Related Work}\label{sec:related}

Normalization has been central to stabilizing deep networks. BatchNorm improves optimization by normalizing intermediate activations \citep{ioffe2015batchnormalizationacceleratingdeep}, while LayerNorm removes batch coupling and became standard in sequence models \citep{ba2016layernormalization}. RMSNorm retains scale normalization but omits mean subtraction, reducing computational cost while preserving rescaling invariance \citep{zhang2019rootmeansquarelayer}. In transformers, the placement of normalization affects both optimization and hidden-state statistics. The pre-norm configuration improves training stability relative to the original post-norm architecture \citep{xiong2020layernormalizationtransformerarchitecture}. More recent work revisits this design choice and proposes Peri-LN, which places normalization around sublayers to better balance variance growth and gradient flow in large-scale transformers \citep{kim2025perilnrevisitingnormalizationlayer}.

A number of approaches reduce reliance on normalization by changing initialization, residual scaling, or parameterization. Fixup trains residual networks without normalization using carefully scaled initialization and learned scalar biases \citep{zhang2019fixupinitializationresiduallearning}. ReZero introduces zero-initialized residual gates and enables training very deep networks, including deep transformers \citep{bachlechner2020rezeroneedfastconvergence}. LayerScale and DeepNorm use residual scaling to improve optimization of deeper transformers \citep{touvron2021goingdeeperimagetransformers,wang2022deepnetscalingtransformers1000}. NF-Nets remove BatchNorm from ConvNets using adaptive gradient clipping and architecture changes \citep{brock2021highperformancelargescaleimagerecognition}. Deep Kernel Shaping and related work investigate signal propagation in networks without skip connections or normalization, including transformer variants \citep{martens2021rapidtrainingdeepneural,he2023deeptransformersshortcutsmodifying}. Finally, $\sigma$Reparam stabilizes transformer training by controlling spectral norms and preventing attention entropy collapse \citep{zhai2023stabilizingtransformertrainingpreventing}. These methods aim to make training stable without standard normalization. In contrast, this work keeps standard pre-norm dynamics early in training and removes selected normalizers gradually.

Several recent works replace normalization layers directly. Dynamic Tanh (DyT) replaces LayerNorm or RMSNorm with a learned elementwise nonlinearity and achieves competitive performance across several transformer settings \citep{zhu2025transformersnormalization}. More recent pointwise replacements such as Derf further investigate which bounded elementwise functions can replace normalization \citep{chen2026strongernormalizationfreetransformers}. These methods remove sample-dependent statistics but leave a nonlinear module in the inference graph. The approach in this work differs in that the tapered normalizer converges to a linear or affine map, which can be absorbed into adjacent projections. As a result, the explicit normalization operation can be removed from the inference graph after training.

Closest to this work is \citet{baroni2025transformersdontneedlayernorm}, who show that LayerNorm layers in GPT-2 can be removed at inference time by fine-tuning and replacing them with linear surrogates. Their method removes LayerNorm blocks sequentially and uses an auxiliary loss that encourages tokenwise standard deviations to be consistent across sequence positions. Our work instead uses a single global gate to taper all selected normalizers at once. It also studies normalization removal during pre-training and separates internal normalization removal from the role of the final normalization layer as an output scale anchor.

The scale-anchoring analysis is related to work on output-scale control. LogitNorm constrains the logit vector norm during training and is motivated by the observation that cross-entropy can increase logit norms and produce overconfident predictions \citep{wei2022mitigatingneuralnetworkoverconfidence}. More generally, gradient descent on separable data with logistic or cross-entropy-type losses can drive predictor norms to infinity while the direction converges to a margin solution \citep{soudry2024implicitbiasgradientdescent}. This work studies the same type of radial pressure at the level of the pre-logit residual stream.

\section{Method}\label{sec:method}

\subsection{Pre-norm transformer and notation}

We consider models with embedding size $d$ and sequence length $T$. Let $H^{(0,0)}\in\mathbb{R}^{T\times d}$ denote the input to the first block. Each layer $j$ has multi-head self-attention (MHSA) and a feedforward network (MLP). Writing $\mathrm{Attn}(H;\theta_{j,\mathrm{A}})$ and $\mathrm{MLP}(H;\theta_{j,\mathrm{M}})$ for these components, the pre-norm residual updates are
\begin{align*}
H^{(j,1)} &= H^{(j,0)} + \mathrm{Attn}\!\left(\mathrm{Norm}^{\mathrm{A}}_j\!\big(H^{(j,0)}\big);\ \theta_{j,\mathrm{A}}\right),\\
H^{(j,2)} &= H^{(j,1)} + \mathrm{MLP}\!\left(\mathrm{Norm}^{\mathrm{M}}_j\!\big(H^{(j,1)}\big);\ \theta_{j,\mathrm{M}}\right),\\
H^{(j+1,0)} &= H^{(j,2)},
\end{align*}
for $j=0,\ldots,L-1$. In the TinyStories pre-training experiments we use SwiGLU activations and RoPE positional encoding.

After the $L$ layers, we apply an output map $\mathrm{Norm}_{\mathrm{final}}$ tokenwise. In our main experiments $\mathrm{Norm}_{\mathrm{final}}$ is a standard RMSNorm or LayerNorm. In ablations that remove output normalization, we implement $\mathrm{Norm}_{\mathrm{final}}$ as a tapered layer whose gate is decayed to $g=0$ over training. We represent token vectors as row vectors, so linear maps act on the right. A learned projection $W_{\mathrm{out}}\in\mathbb{R}^{d\times V}$ yields logits
\[
z \;=\; \mathrm{Norm}_{\mathrm{final}}\!\big(H^{(L,0)}\big)\,W_{\mathrm{out}} \in \mathbb{R}^{T\times V}.
\]
For a token at position $t$, let $z_t\in\mathbb{R}^V$ denote the $t$-th row of $z$ and let $y_t$ be its target id.
The token-level cross-entropy is $\ell(z_t,y_t) = -\log \mathrm{softmax}(z_t)_{y_t}$.

\subsection{LayerNorm and RMSNorm}

Given token row vector $h\in\mathbb{R}^{1\times d}$, LayerNorm and RMSNorm are
\begin{equation}
\begin{aligned}
\mathrm{LayerNorm}(h) &= \frac{h - \mu_h}{\sigma_h}\,D_\gamma + \beta, \quad
\mu_h=\tfrac{1}{d}\sum_{i=1}^d h_i,&\quad\sigma_h=\sqrt{\tfrac{1}{d}\sum_{i=1}^d (h_i-\mu_h)^2+\varepsilon},\\
\end{aligned}
\end{equation}
\begin{equation}
\begin{aligned}
\mathrm{RMSNorm}(h)   &= \frac{h}{r(h)}\,D_\gamma, \quad r(h)&=\sqrt{\|h\|_2^2/d+\varepsilon},
\end{aligned}
\end{equation}
with learnable per-feature scales $\gamma \in \mathbb{R}^d$ (and bias $\beta \in \mathbb{R}^d$ for LayerNorm).
We write $D_\gamma \coloneqq \mathrm{diag}(\gamma)$ and $D_{\tilde\gamma} \coloneqq \mathrm{diag}(\tilde\gamma)$ for diagonal scalings.

\subsection{TaperNorm}\label{subsec:taper}

TaperNorm dynamically transitions from RMSNorm/LayerNorm to a fixed linear or affine map of features. For gate $g\in[0,1]$ and token $h\in\mathbb{R}^{1\times d}$,
\begin{equation}\label{eq:taper}
\mathrm{TaperNorm}(h;g)
\;=\;
g\,\frac{h}{r(h)}\,D_\gamma
\;+\;
(1-g)\,c\,h\,D_{\tilde\gamma},
\end{equation}
where $r(h)$ is as in RMSNorm, $c\in\mathbb{R}$ is a per-layer scalar, and $\tilde\gamma\in\mathbb{R}^d$ is a per-feature gain that remains trainable during tapering. The LayerNorm variant, used in GPT-2 experiments, is analogous and can be found in Appendix~\ref{app:TaperLN}. For $g=1$ we recover RMSNorm/LayerNorm. For $g=0$ we obtain a sample-independent scaling that can be fused into downstream linear operations.

We keep $g(k)=1$ during an initial gate warmup period to obtain standard normalized dynamics and to accumulate EMA statistics used to set $c$. At the end of gate warmup (the taper start), each TaperNorm layer computes its own $c^\star$ from $\gamma$-weighted EMAs (Section~\ref{subsec:cstar}), freezes this scalar for the remainder of training, and initializes $\tilde\gamma\leftarrow\gamma$. Both $\gamma$ and $\tilde\gamma$ remain trainable and continue to update until convergence. After taper start, $g(k)$ is cosine-decayed from $1$ to $0$. This same scalar $g(k)$ is applied to every TaperNorm layer that is being tapered. At convergence, $g=0$ and per-token statistics are no longer used.

When $g=0$, TaperNorm can be removed from the inference graph by weight folding. For the RMSNorm case, $\mathrm{TaperNorm}(h;0)=c\,h\,D_{\tilde\gamma}$. If the next operation is a linear map $x\mapsto xW$ (with an optional bias), then we can replace $W$ by $c\,D_{\tilde\gamma}W$ and remove the explicit TaperNorm. The same folding applies to attention and MLP input projections. For the LayerNorm variant, the $g=0$ map is affine and can be folded into a following linear layer as described in Appendix~\ref{app:TaperLN}.

The tapering procedure can be applied to any subset of normalization layers. This work primarily applies it to internal normalizers. This setting gives most of the inference benefit because internal normalizers are repeated throughout the transformer blocks, while the final normalization is applied only once before the output projection. It also keeps the final normalization as an implicit scale anchor. All-Taper ablations are reported to investigate the fully norm-free setting and to test whether an explicit scale anchor can substitute for the final normalization.

\subsection{Calibrating the scaling branch}\label{subsec:cstar}

Each tapered normalizer interpolates between a normalized branch and a sample-independent scaling branch. When the gate begins to decay after warmup, we would like the two branches to agree as closely as possible so the interpolation does not introduce an avoidable distribution shift. We therefore align the scaling branch to the normalization branch at the warmup boundary by setting $D_{\tilde\gamma}\leftarrow D_\gamma$ and choosing a single scalar $c$ that best matches normalized activations in least squares
\begin{equation}\label{eq:cstar}
\begin{aligned}
c^\star \in& \arg\min_{c}
\mathbb{E}\!\left[\left\|\tfrac{h}{r(h)}D_\gamma - c\,hD_\gamma\right\|_2^2\right]
= \frac{\mathbb{E}\!\left[\|hD_\gamma\|_2^2/r(h)\right]}
        {\mathbb{E}\!\left[\|hD_\gamma\|_2^2\right]}\,.
\end{aligned}
\end{equation}
This matches the two inputs in energy under the layer's current scaling, not in raw coordinates. The derivation is given in Appendix~\ref{app:proofs:cstar}. In implementation, the expectations in Eq.~\eqref{eq:cstar} are estimated during warmup using EMAs of the numerator and denominator. At taper start, bias correction is applied, $c$ is set to the resulting ratio, and $\gamma$ is copied to $\tilde\gamma$. The scalar $c$ is then frozen, while $\gamma$ and $\tilde\gamma$ remain trainable.

\subsection{Fixed-target scale anchoring loss}\label{subsec:aux}

The fixed-target scale loss is separate from the tapering mechanism. It is used to anchor the scale of the pre-logit residual stream. In Internal-Taper, the final normalization remains the main output scale anchor. In All-Taper, where the final normalization is also removed, the fixed-target loss provides an explicit alternative anchor. Let token vectors before the final normalization, rows of $H^{(L,0)}$, be denoted using $h_{t}$. Define a per-token scale statistic
\[
s(h)=
\begin{cases}
r(h)=\sqrt{\|h\|_2^2/d+\varepsilon} & \text{\small for RMSNorm} ,\\
\sigma_h=\sqrt{\tfrac{1}{d}\sum_{i=1}^d (h_i-\mu_h)^2+\varepsilon} & \text{\small for LayerNorm} .
\end{cases}
\]
The auxiliary loss penalizes deviations of token scales from a single scalar target
\begin{equation}\label{eq:aux-fixed}
\mathcal{L}_{\mathrm{aux}}
\;=\;
\lambda\,\mathbb{E}_{b,t}\!\left( s(h_{b,t}) - s_{\mathrm{tgt}} \right)^2,
\end{equation}
where $(b,t)$ index batch elements and sequence positions, and the expectation is taken over non-padding tokens. While $g=1$ (gate warmup) we track an EMA of the batch mean $\hat s_k=\mathbb{E}_{b,t}[s(h_{b,t})]$, and at taper start we freeze $s_{\mathrm{tgt}}$ to the bias-corrected EMA value. After taper start, $s_{\mathrm{tgt}}$ remains constant for the remainder of training. We train with the combined objective $\mathcal{L}=\mathcal{L}_{\mathrm{CE}}+\mathcal{L}_{\mathrm{aux}}$, where $\mathcal{L}_{\mathrm{CE}}$ is token-level cross-entropy.

\section{Scale anchoring}\label{sec:theory}

This section analyzes one role of normalization that becomes important when the final normalization layer is removed, namely, control of the scale of the pre-logit representation. When the final normalization is removed, the scale of the last hidden state becomes a degree of freedom that can drift during training. A final normalization layer anchors it implicitly, while the fixed-target scale loss anchors it explicitly.

We analyze a single token vector $h\in\mathbb{R}^{1\times d}$ that is input to the output projection. The logits are
\[
z = \mathrm{Norm}_{\mathrm{final}}(h)\,W_{\mathrm{out}} \in \mathbb{R}^{1\times V}.
\]
Only the radial component $\langle \nabla_h \ell(z,y), h\rangle$ can change $\|h\|_2$ to first order. The propositions below characterize this radial component for three cases that correspond to the design choices in our method. Proofs can be found in Appendix~\ref{app:proofs}.

\begin{proposition}[Final normalization removes radial gradient]\label{prop:tangent}
Let the final map $\mathrm{Norm}_{\mathrm{final}}$ before the output projection be $0$-homogeneous and differentiable almost everywhere. This includes the idealized RMSNorm and LayerNorm maps with $\varepsilon=0$. For logits $z=\mathrm{Norm}_{\mathrm{final}}(h)\,W_{\mathrm{out}}$ and any differentiable loss $\ell(z,y)$,
\[
\langle \nabla_h\,\ell(z,y),\,h\rangle = 0 \quad (h\neq 0).
\]
\end{proposition}

Proposition~\ref{prop:tangent} shows that normalization removes the radial gradient at the output. In this setting the loss cannot be reduced by simply rescaling the last hidden state.

\begin{proposition}[Without the final norm, cross-entropy pushes norms up]\label{prop:push}
If $z=h\,W_{\mathrm{out}}$ and $\ell$ is multiclass cross-entropy, then whenever the margin $m=z_y-\max_{j\neq y}z_j>0$,
\[
\langle \nabla_h\ell(z,y),h\rangle
=\langle \nabla_z\ell,z\rangle
\le -(1-\mathrm{softmax}(z)_y)m < 0,
\]
so a small gradient step along $-\nabla_h\ell$ increases $\|h\|_2^2$ to first order.
\end{proposition}

Proposition~\ref{prop:push} captures a simple mechanism for unbounded scale growth when the final normalization is removed. Once the correct class has positive margin, scaling logits along their current direction reduces cross-entropy. We refer to this tendency toward increasing confidence through logit magnitude as logit chasing.

\begin{proposition}[Fixed-target scale loss provides a radial restoring force]\label{prop:aux-anchor}
Let $r(h)=\sqrt{\|h\|_2^2/d+\varepsilon}$ and define the fixed-target auxiliary loss on a token vector $h$ by
\[
\mathcal{L}_{\mathrm{aux}}(h)=\lambda\,(r(h)-s_{\mathrm{tgt}})^2,
\qquad \lambda>0.
\]
Then
\begin{align*}
  \nabla_h \mathcal{L}_{\mathrm{aux}}(h)
=\frac{2\lambda\,(r(h)-s_{\mathrm{tgt}})}{d\,r(h)}\,h, \quad
\langle \nabla_h \mathcal{L}_{\mathrm{aux}}(h),\,h\rangle
=\frac{2\lambda\,(r(h)-s_{\mathrm{tgt}})}{d\,r(h)}\,\|h\|_2^2.  
\end{align*}
In particular, if $r(h)>s_{\mathrm{tgt}}$ then a small gradient step on $\mathcal{L}_{\mathrm{aux}}$ decreases $\|h\|_2$
to first order, while if $r(h)<s_{\mathrm{tgt}}$ it increases $\|h\|_2$.
\end{proposition}

Proposition~\ref{prop:aux-anchor} shows that the auxiliary loss injects an explicit restoring force on scale. For LayerNorm with $s(h)=\sigma_h$, the same conclusion holds with the restoring direction given by the mean-centered vector. The derivative formulas are given in Appendix~\ref{app:proofs:ln-scale}.

Taken together, Propositions~\ref{prop:tangent}--\ref{prop:aux-anchor} motivate our default training choices. Keeping the final normalization provides an implicit anchor that removes the radial component at the output. If we also taper away the final normalization, the fixed-target scale loss provides an explicit alternative that counteracts the radial push from cross-entropy.

\section{Experiments}\label{sec:experiments}

This section evaluates the normalization-removal approach in three settings. First, we pre-train pre-norm transformers on TinyStories to test whether internal normalization can be tapered when training a small model from scratch. Second, we run a KV-cached decoding microbenchmark to measure the effect of folding the resulting sample-independent scalings into adjacent projections. Third, we fine-tune GPT-2 models using the LayerNorm variant and compare to prior LayerNorm removal results.

\subsection{Pre-training}\label{subsec:pretrain}

We test pre-training performance on TinyStories with causal, pre-norm transformers (depth 8, 16 attention heads at all widths, SwiGLU MLPs, RoPE) while varying width to obtain models at a range of scales from $1$ to $30$ million parameters. We repeat each configuration with 6 random seeds and report mean and standard deviation. Evaluation of tapered models is performed with the gate set to 0. Shaded regions in the plots indicate the $2.5$--$97.5$ percentile range across seeds for each configuration.

All pre-training runs use a SentencePiece tokenizer trained on TinyStories (vocabulary size $10$k) and a context length of $512$ tokens. We use AdamW with peak learning rate $3\times10^{-4}$, $\beta=(0.9,0.95)$, $5\%$ linear warmup, and cosine decay to zero. We use no dropout or weight decay in these experiments.

The default tapering setting includes the fixed-target scale anchoring loss from Section~\ref{subsec:aux}. For ablations we also report runs without it. Among the taper variants, Internal-Taper (+aux) with the final RMSNorm retained is the most stable and reliable across seeds, and we use it as the default configuration in our main comparisons. We additionally report All-Taper (+aux) to quantify the fully norm-free setting enabled by explicit scale anchoring.

For all tapered models we follow this gate schedule. During learning-rate warmup the gate is fixed at $g=1$. At the warmup boundary we compute $c^\star$ from $\gamma$-weighted EMAs with bias correction, copy $\gamma \to \tilde\gamma$, and then freeze $c^\star$. Post-warmup, $g$ is cosine-decayed to $0$. All other training details between same-size models are identical.

Figure~\ref{fig:loss} shows evolution of training loss for the Baseline and Internal-Taper (+aux). It can be seen that, with the final normalization kept in place, the trajectories are very similar. The main difference appears near the end of training, when the gate has reached zero and per-token statistics have been removed from the internal layers. Table~\ref{tab:scales} reports final validation losses. Internal-Taper (+aux) is within about $0.7$--$1.5\%$ relative loss of the Baseline, while All-Taper (+aux) is within about $0.8$--$1.8\%$.

\begin{figure}[t]
  \centering
  \begin{subfigure}[t]{.48\linewidth}
    \centering
    \includegraphics[width=0.8\linewidth]{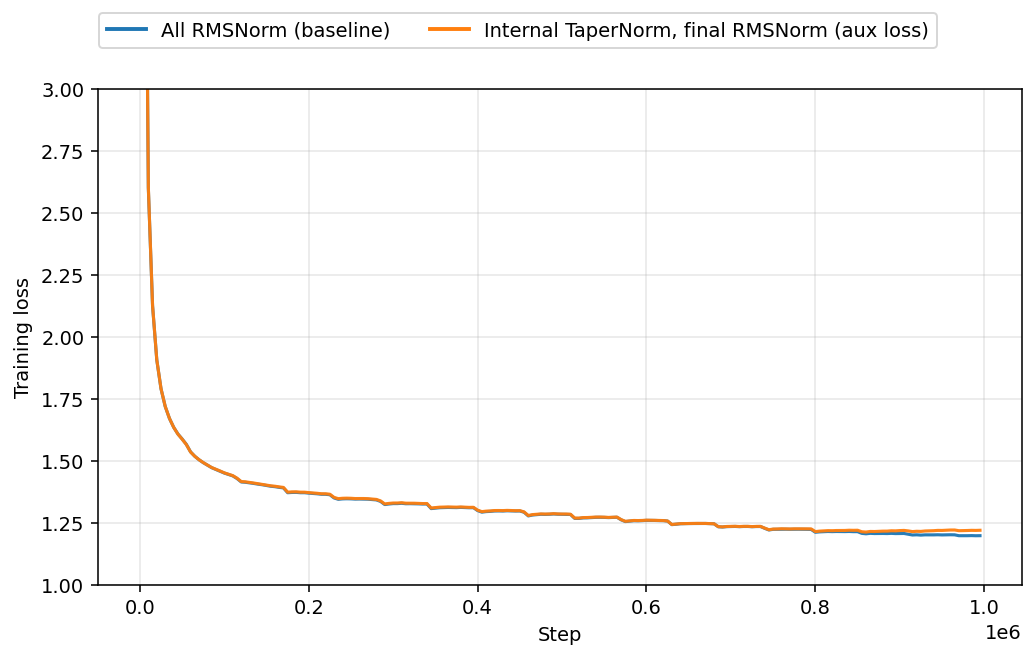}
    \caption{Training loss.}
    \label{fig:loss}
  \end{subfigure}\hfill
  \begin{subfigure}[t]{.48\linewidth}
    \centering
    \includegraphics[width=0.8\linewidth]{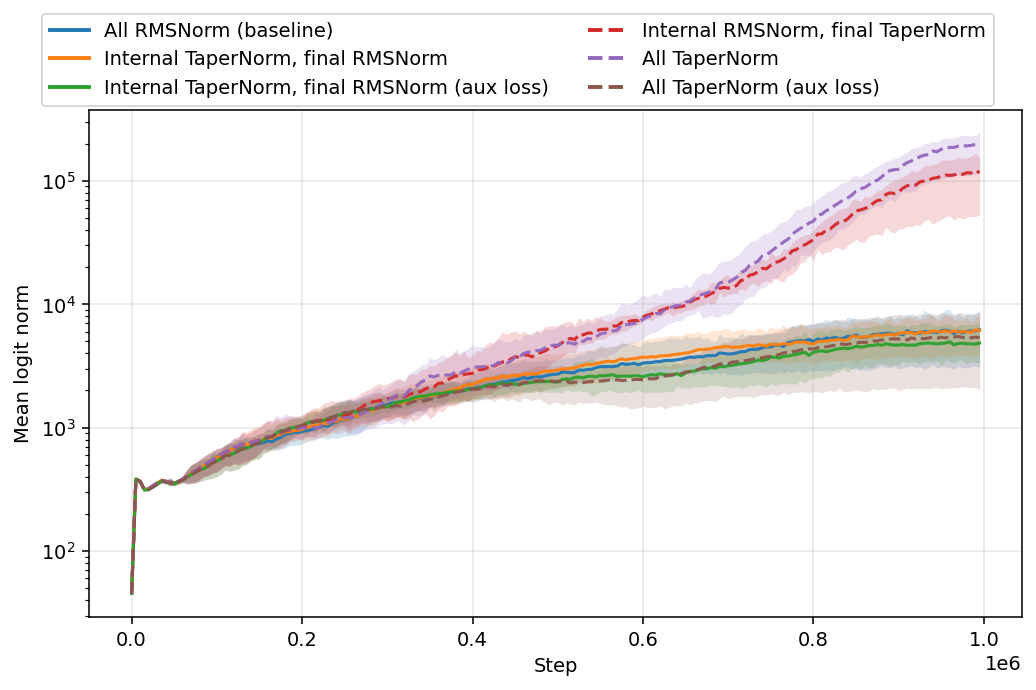}
    \caption{Mean logit $\ell_2$ norm.}
    \label{fig:logitl2}
  \end{subfigure}
  \vspace{-0.5em}
  \caption{Training dynamics and output-scale behavior. Left: training loss vs.\ step for the Baseline and Internal-Taper (+aux). Right: without a scale anchor, models with the final norm removed (dashed) can exhibit logit chasing, consistent with Proposition~\ref{prop:push}; with the fixed-target scale loss, logit growth is strongly suppressed in the All-Taper setting.}
  \label{fig:loss-logitl2}
\end{figure}

\begin{table}[t]
\centering
\small
\renewcommand{\arraystretch}{0.8}
\caption{Validation loss for Baseline, Internal-Taper (+aux), and All-Taper (+aux). Percentages in parentheses show the relative gap vs.\ the Baseline. $^\dagger$ For 9M All-Taper (+aux), we report mean$\pm$std over 5 seeds after excluding one outlier run with unusually high validation loss.}
\begin{tabular}{@{}rccc@{}}
\toprule
Model size & Baseline & \shortstack{Internal-Taper (+aux)} & \shortstack{All-Taper (+aux)} \\
\midrule
$1$M  & 2.1538$\pm$0.0025
      & \shortstack{2.1852$\pm$0.0040 (+1.46\%)}
      & \shortstack{2.1930$\pm$0.0035 (+1.82\%)} \\
$3$M  & 1.7561$\pm$0.0010
      & \shortstack{1.7828$\pm$0.0029 (+1.52\%)}
      & \shortstack{1.7874$\pm$0.0028 (+1.78\%)} \\
$9$M  & 1.4620$\pm$0.0022
      & \shortstack{1.4837$\pm$0.0016 (+1.48\%)}
      & \shortstack{1.4837$\pm$0.0023$^\dagger$(+1.48\%)} \\
$30$M & 1.2768$\pm$0.0018
      & \shortstack{1.2859$\pm$0.0010 (+0.71\%)}
      & \shortstack{1.2868$\pm$0.0041 (+0.78\%)} \\
\bottomrule
\end{tabular}
\label{tab:scales}
\end{table}

Figures~\ref{fig:logitl2} and~\ref{fig:gradgrid} show the effect of output scale anchoring. Without an explicit anchor, removing the final normalization layer causes rapid growth in logit magnitudes, consistent with Proposition~\ref{prop:push}. It can also be seen that gradient magnitudes cluster according to the presence or absence of the final normalization. With the fixed-target scale loss enabled, the All-Taper model has gradient magnitudes similar to the Internal-Taper model. This indicates that explicit scale anchoring can replace part of the stabilizing effect of final normalization in this setting.

\begin{figure}[t]
  \centering
  \includegraphics[width=.95\linewidth]{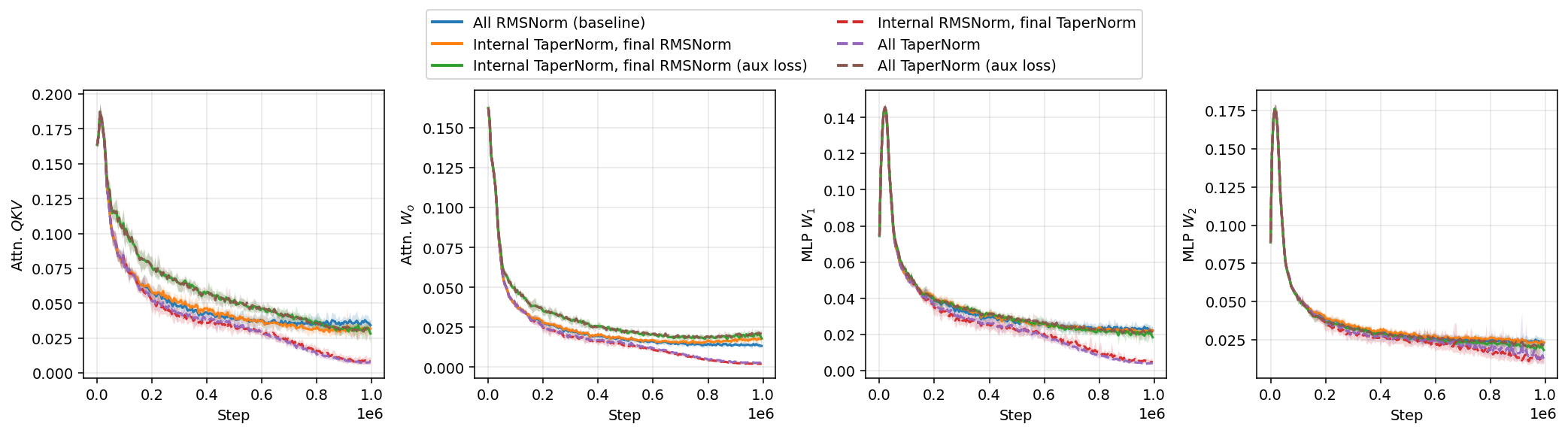}
  \vspace{-0.5em}
  \caption{Average gradient norms across all transformer blocks by weight type. Without explicit scale anchoring, gradients cluster primarily by presence vs.\ absence of the final normalization. With the fixed-target scale loss enabled, the gradient-magnitude gap between Internal-Taper and All-Taper largely disappears.}
  \label{fig:gradgrid}
\end{figure}

\subsection{KV-cached decoding microbenchmark}\label{subsec:efficiency}

A practical consequence of tapering normalization is that, once $g\rightarrow 0$, each tapered layer becomes sample-independent and can be removed from the inference graph by weight folding. We measure this effect on the TinyStories pre-trained Internal-Taper (+aux) $d=512$ model ($\approx30$M parameters). Internal-Taper is used because it is more stable than All-Taper while providing nearly identical computational savings from internal normalization removal.

We benchmark autoregressive generation with standard KV caching on a single NVIDIA H100 80GB (CUDA 12.1, PyTorch 2.4.1+cu121) using bf16. We enable TF32 matmul where applicable. For each setting, a prompt of length $T$ is provided and tokens are generated one at a time while reusing the KV cache. Throughput is reported in tokens per second (tok/s).

We compare the following settings.
\begin{itemize}
    \item \textbf{Baseline} (RMSNorm).
    \item \textbf{Internal-Taper (unfused)} where the tapered normalizers are exported to fixed elementwise scalings but left as explicit operations.
    \item \textbf{Internal-Taper (fused)} where the internal scalings are folded into the attention QKV and MLP input projections, allowing normalization modules to be replaced by identity.
\end{itemize}

\begin{table}[t]
\centering
\small
\renewcommand{\arraystretch}{0.8}
\caption{Autoregressive decoding throughput (tok/s, higher is better) with KV caching on an NVIDIA H100 80GB (bf16, PyTorch 2.4.1+cu121, CUDA 12.1). Parentheses show speedup relative to the Baseline RMSNorm model.}
\begin{tabular}{@{}ccccc@{}}
\toprule
Batch $B$ & Prompt $T$ & Baseline & \shortstack{Internal-Taper\\(unfused)} & \shortstack{Internal-Taper\\(fused)} \\
\midrule
1 & 128 & 336 & 383 (1.14$\times$) & 397 (1.18$\times$) \\
1 & 256 & 340 & 383 (1.13$\times$) & 398 (1.17$\times$) \\
4 & 128 & 1{,}253 & 1{,}382 (1.10$\times$) & 1{,}441 (1.15$\times$) \\
4 & 256 & 1{,}250 & 1{,}378 (1.10$\times$) & 1{,}437 (1.15$\times$) \\
\bottomrule
\end{tabular}
\label{tab:efficiency}
\end{table}

Table~\ref{tab:efficiency} reports KV-cached decoding throughput. It can be seen that the unfused Internal-Taper model gives a $1.10$--$1.14\times$ improvement over the RMSNorm baseline, while folding the internal scaling into adjacent projections gives a $1.15$--$1.18\times$ improvement. These results indicate that normalization removal gives measurable gains even when KV caching is used. The measured speedups remain specific to the model size, hardware, software stack, batch size, and prompt lengths tested.

\subsection{Fine-tuning GPT-2 with gated LayerNorm removal}\label{sec:ft-gpt2}

We fine-tune GPT-2 family models and compare to the LayerNorm removal results of \citet{baroni2025transformersdontneedlayernorm}. They show that LayerNorm layers in GPT-2 can be removed at inference time via fine-tuning with only a small increase in validation loss. Their key mechanism is to replace each LayerNorm with a linear surrogate, \emph{FakeLN}, in which the per-token standard deviation $\sigma$ is replaced by a single scalar $\sigma_{\mathrm{avg}}$ computed as the average of tokenwise standard deviations
\[
\mathrm{FakeLN}(x)=\frac{x-\mu}{\sigma_{\mathrm{avg}}}D_{\gamma}+\beta.
\]
Because removing all LayerNorm blocks simultaneously destabilizes the model, they adopt a sequential removal protocol and add an auxiliary loss that encourages tokenwise standard deviations to be consistent across sequence positions. This term is closely related to our scale-based view of normalization, but its target is batch-dependent and therefore does not pin an absolute pre-logit scale. Our fixed-target loss instead anchors the pre-logit residual stream to a single absolute scale, which better aligns with the tapering approach since the scaling branch learns its own gain. Additional details on this comparison are given in Appendix~\ref{app:ft-details}.

We compare four settings.
\begin{itemize}
    \item \textbf{Vanilla \citep{baroni2025transformersdontneedlayernorm}.} Standard fine-tuned checkpoints with all LayerNorm layers intact.
    \item \textbf{FakeLN \citep{baroni2025transformersdontneedlayernorm}.} We use the LayerNorm-free checkpoints published by the authors and evaluate them using their released code.
    \item \textbf{Internal-TaperLN (ours, +fixed-target aux).} We taper every internal LayerNorm using TaperLN (Appendix~\ref{app:TaperLN}) and cosine-decay the shared gate to $g=0$, while keeping the final LayerNorm active (locked at $g=1$).
    \item \textbf{All-TaperLN (ours, +fixed-target aux).} We taper all LayerNorm layers, including the final LayerNorm, using TaperLN and cosine-decay the shared gate to $g=0$, yielding a fully norm-free model at inference.
\end{itemize}
We fine-tune on OpenWebText \cite{Gokaslan2019OpenWeb}, reuse the hyperparameters from \citet{baroni2025transformersdontneedlayernorm} and set the tapering window by approximately aligning the end of warmup with their first LayerNorm removal step and the end of cosine gate decay with their first LayerNorm-free checkpoint. The exact schedule used is provided in Appendix~\ref{app:ft-details}.

Table~\ref{tab:ft-end} reports end-of-run cross-entropy on OpenWebText, The Pile \cite{gao2020pile}, and The Pile-filtered. Across GPT-2 Small, Medium, Large, and XL, Internal-TaperLN achieves losses in the same range as the LayerNorm-free FakeLN checkpoints while keeping the final LayerNorm as an implicit scale anchor. All-TaperLN yields a fully norm-free model and remains competitive in several settings, with a modest increase in others. In contrast to the staged FakeLN removal procedure, our taper protocols use a single global cosine gate.

\begin{table}[t]
\centering
\scriptsize
\renewcommand{\arraystretch}{0.7}
\caption{End metrics on OpenWebText (OWT), The Pile, and The Pile-filtered for GPT-2 S/M/L/XL.
Vanilla and FakeLN (+Baroni-aux) results follow \citet{baroni2025transformersdontneedlayernorm}.
$^\dagger$ \citet{baroni2025transformersdontneedlayernorm} report that for GPT-2 XL LayerNorm free on The Pile, the mean CE loss can be dominated by a very small number of outlier sequences.}
\begin{tabularx}{\columnwidth}{@{}Lccc@{}}
\toprule
Model/Protocol & OWT & Pile & \shortstack{Pile-filtered}\\
\midrule
S/Vanilla & 3.0126 & 2.8511 & 2.8112 \\
S/FakeLN (+Baroni-aux) & 3.0797 & 2.8852 & 2.8757 \\
S/Internal-TaperLN (+fixed-target aux) & 3.0584 & 2.8381 & 2.8374 \\
S/All-TaperLN (+fixed-target aux) & 3.0685 & 2.8326 & 2.8416 \\
\midrule
M/Vanilla & 2.7390 & 2.5752 & 2.5724 \\
M/FakeLN (+Baroni-aux) & 2.7642 & 2.6579 & 2.6352 \\
M/Internal-TaperLN (+fixed-target aux) & 2.7550 & 2.5937 & 2.5936 \\
M/All-TaperLN (+fixed-target aux) & 2.7715 & 2.6248 & 2.6191 \\
\midrule
L/Vanilla & 2.6240 & 2.6233 & 2.5074 \\
L/FakeLN (+Baroni-aux) & 2.6384 & 2.7504 & 2.5159 \\
L/Internal-TaperLN (+fixed-target aux) & 2.6277 & 2.7860 & 2.4910 \\
L/All-TaperLN (+fixed-target aux) & 2.6230 & 2.6539 & 2.4814 \\
\midrule
XL/Vanilla & 2.4799 & 2.4673 & 2.3821 \\
XL/FakeLN (+Baroni-aux) & 2.5052 & 130.2197$^\dagger$ & 2.3992 \\
XL/Internal-TaperLN (+fixed-target aux) & 2.5144 & 2.5303 & 2.3844 \\
XL/All-TaperLN (+fixed-target aux) & 2.5305 & 2.6504 & 2.4000 \\
\bottomrule
\end{tabularx}
\label{tab:ft-end}
\end{table}

\section{Discussion}\label{sec:discussion}

The results indicate that internal and final normalization play different roles in pre-norm transformers. Internal-Taper removes the repeated normalizers inside transformer blocks while keeping the final normalization active. This is the most stable setting in our experiments and provides nearly all of the computational benefit from folding internal scalings. All-Taper removes the final normalization as well, but depends more strongly on explicit scale anchoring.

The final normalization acts as an output scale anchor. It controls the radial degree of freedom of the pre-logit representation because normalization removes radial gradients at the output (Proposition~\ref{prop:tangent}). Without an anchor, cross-entropy induces logit chasing (Proposition~\ref{prop:push}). The fixed-target scale loss provides an explicit alternative anchor by introducing a radial restoring force (Proposition~\ref{prop:aux-anchor}).

This view explains the main training behavior observed in the ablations. When the final normalization is removed without an explicit anchor, logit norms grow and gradient magnitudes separate according to the presence of the final normalization. When the fixed-target scale loss is enabled, this separation largely disappears. This demonstrates that explicit scale anchoring can replace part of the stabilizing role of final normalization in the tested setting.

\section{Limitations and future work}\label{sec:limitations}

The pre-training experiments are limited to TinyStories models up to $30$M parameters with depth $8$. They do not establish that the same tapering schedule and calibration procedure transfer to billion-parameter pre-training, deeper models, or longer contexts. Testing the method across depth, width, data scale, and training duration is an important direction for future work.

The approach ultimately represents selected sample-dependent normalization layers as linear or affine maps. As a result, robustness to activation outliers and distribution shift remains an open question. Although the model continues training after the scaling branch is introduced, the experiments do not test the large-scale outlier regimes observed in modern language models.

The fixed-target scale loss is studied for cross-entropy language modeling. Other objectives may interact differently with output scale. The recipe also contains schedule and loss-weight choices.

The efficiency results are microbenchmarks on a single hardware and software stack and measure KV-cached autoregressive decoding for the $30$M TinyStories model. End-to-end serving speedups may differ with model size, kernel availability, attention implementation, batching, and runtime optimizations.

\section{Conclusion}\label{sec:conclusion}

This work introduced a gated normalization-removal approach that starts from standard RMSNorm or LayerNorm behavior and transitions selected normalizers to learned sample-independent rescalings. The analysis identifies output scale anchoring as a distinct role of the final normalization layer: it removes the radial output gradient, while without an anchor cross-entropy induces logit chasing. A fixed-target scale loss provides an explicit alternative anchor and suppresses this behavior when the final normalization is also tapered away. Empirically, internal tapering gives modest validation-loss gaps on TinyStories, competitive GPT-2 LayerNorm-removal results, and up to $1.18\times$ higher H100 KV-cached decoding throughput after folding internal scalings.

\bibliographystyle{plainnat}
\bibliography{bibliography}

\newpage
\appendix
\section{Proofs for Section~\ref{sec:theory}}\label{app:proofs}

\subsection*{Preliminaries and assumptions}
We recall $r(h)\coloneqq \sqrt{\|h\|_2^2/d+\varepsilon}$ and $D_\gamma\coloneqq \mathrm{diag}(\gamma)$, $D_{\tilde\gamma}\coloneqq \mathrm{diag}(\tilde\gamma)$. Expectations $\mathbb{E}[\cdot]$ are over mini-batch elements and sequence positions unless stated.

\subsection{Proof of Proposition~\ref{prop:tangent}}\label{app:proofs:prop1}
\begin{proof}
A map $\mathrm{Norm}$ is $0$-homogeneous if $\mathrm{Norm}(\alpha h)=\mathrm{Norm}(h)$ for all $\alpha>0$ and $h\neq 0$. Differentiating in $\alpha$ at $\alpha=1$ gives $\frac{\mathrm{d}}{\mathrm{d}\alpha}\mathrm{Norm}(\alpha h)\big|_{\alpha=1}=h\,J_{\mathrm{Norm}}(h)=0$, so $h$ lies in the left nullspace of the Jacobian $J_{\mathrm{Norm}}(h)$.
With $z=\mathrm{Norm}_{\mathrm{final}}(h)\,W_{\mathrm{out}}$ we have by the chain rule
\[
\nabla_h \ell(z,y) \;=\; \nabla_z \ell(z,y)\,W_{\mathrm{out}}^\top\,J_{\mathrm{Norm}}(h)^\top.
\]
Taking the inner product with $h$,
\[
\langle \nabla_h \ell, h\rangle
= \nabla_z \ell\,W_{\mathrm{out}}^\top\,J_{\mathrm{Norm}}(h)^\top\,h^\top
= \nabla_z \ell\,W_{\mathrm{out}}^\top\,(h\,J_{\mathrm{Norm}}(h))^\top
=0.
\]
The identity is exact for exactly $0$-homogeneous maps. RMSNorm and LayerNorm satisfy this property in the limit $\varepsilon=0$.
\end{proof}

\subsection{Proof of Proposition~\ref{prop:push}}\label{app:proofs:prop2}
\begin{proof}
For multiclass cross-entropy with softmax probabilities $p=\mathrm{softmax}(z)$, $\nabla_z \ell = p - e_y$. Then
\[
\langle \nabla_z \ell, z\rangle \;=\; \sum_i z_i(p_i-1_{i=y}) \;=\; \sum_i p_i z_i - z_y \;=\; \mathbb{E}_p[z] - z_y.
\]
Let $m=z_y-\max_{j\neq y} z_j>0$ and $b=\max_{j\neq y} z_j \le z_y - m$. Since $\sum_{j\neq y}p_j=1-p_y$ and $z_j\le b$ for $j\neq y$,
\[
\mathbb{E}_p[z] \;=\; p_y z_y + \sum_{j\neq y} p_j z_j \;\le\; p_y z_y + (1-p_y) b \;\le\; p_y z_y + (1-p_y)(z_y-m) \;=\; z_y - (1-p_y)m.
\]
Hence $\langle \nabla_z \ell, z\rangle \le -(1-p_y)m < 0$. With $z=h\,W_{\mathrm{out}}$ and $\nabla_h \ell = \nabla_z \ell\,W_{\mathrm{out}}^\top$, we have
$\langle \nabla_h \ell, h\rangle = \langle \nabla_z \ell, z\rangle$.
A gradient descent step $h^+=h-\eta \nabla_h \ell$ therefore satisfies
\[
\|h^+\|_2^2 - \|h\|_2^2 \;=\; -2\eta\,\langle \nabla_h \ell, h\rangle + O(\eta^2) \;\ge\; 2\eta\,(1-p_y)m + O(\eta^2),
\]
so the norm increases to first order.
\end{proof}

\subsection{Proof of Proposition~\ref{prop:aux-anchor}}\label{app:proofs:aux-anchor}
\begin{proof}
Let $r(h)=\sqrt{\|h\|_2^2/d+\varepsilon}$. Since $\nabla_h \|h\|_2^2 = 2h$, we have
\[
\nabla_h r(h)
= \frac{1}{2}\left(\|h\|_2^2/d+\varepsilon\right)^{-1/2}\cdot \frac{2h}{d}
= \frac{h}{d\,r(h)}.
\]
Therefore
\[
\nabla_h \mathcal{L}_{\mathrm{aux}}(h)
= 2\lambda\,(r(h)-s_{\mathrm{tgt}})\,\nabla_h r(h)
= \frac{2\lambda\,(r(h)-s_{\mathrm{tgt}})}{d\,r(h)}\,h.
\]
Taking the inner product with $h$ yields the second identity, and the final statement follows from the first-order norm expansion $\|h-\eta\nabla_h\mathcal{L}_{\mathrm{aux}}\|_2^2=\|h\|_2^2-2\eta\langle\nabla_h\mathcal{L}_{\mathrm{aux}},h\rangle+O(\eta^2)$.
\end{proof}

\subsection{LayerNorm scale gradient formulas}\label{app:proofs:ln-scale}

Let $h\in\mathbb{R}^{1\times d}$ be a token vector, let $\mu_h=\tfrac{1}{d}\sum_{i=1}^d h_i$, and define the mean-centered vector $\bar h \coloneqq h-\mu_h\mathbf{1}$. Let
\[
\sigma_h=\sqrt{\tfrac{1}{d}\|\bar h\|_2^2+\varepsilon}.
\]
Then
\[
\nabla_h \sigma_h \;=\; \frac{\bar h}{d\,\sigma_h}.
\]
Therefore, for the fixed-target auxiliary loss $\lambda(\sigma_h-s_{\mathrm{tgt}})^2$,
\[
\nabla_h \big(\lambda(\sigma_h-s_{\mathrm{tgt}})^2\big)
\;=\;
\frac{2\lambda(\sigma_h-s_{\mathrm{tgt}})}{d\,\sigma_h}\,\bar h.
\]
This shows that the auxiliary term pulls $\sigma_h$ toward $s_{\mathrm{tgt}}$ through a restoring force in the mean-zero subspace.

\subsection{Derivation of $c^\star$ in~\eqref{eq:cstar}}\label{app:proofs:cstar}
Let $A(h)=\tfrac{h}{r(h)}D_\gamma$ and $B(h)=hD_\gamma$. The objective is
\[
J(c)=\mathbb{E}\big[\|A(h)-c\,B(h)\|_2^2\big]=\mathbb{E}\big[\|A\|_2^2\big]-2c\,\mathbb{E}\big[\langle A,B\rangle\big]+c^2\mathbb{E}\big[\|B\|_2^2\big].
\]
Differentiating and setting to zero gives
\[
0 = J'(c) = -2\,\mathbb{E}\big[\langle A,B\rangle\big] + 2c\,\mathbb{E}\big[\|B\|_2^2\big]
\quad\Rightarrow\quad
c^\star = \frac{\mathbb{E}\big[\langle A,B\rangle\big]}{\mathbb{E}\big[\|B\|_2^2\big]}= \frac{\mathbb{E}\!\left[\|hD_\gamma\|_2^2/r(h)\right]}{\mathbb{E}\!\left[\|hD_\gamma\|_2^2\right]}.
\]

\subsection{EMA estimators for $c^\star$}\label{app:proofs:ema}
During warmup we estimate
\[
\overline{a}=\mathbb{E}\big[\|hD_\gamma\|_2^2/r(h)\big],\qquad
\overline{b}=\mathbb{E}\big[\|hD_\gamma\|_2^2\big]
\]
with exponential moving averages (EMAs) $s_{\mathrm{num}},s_{\mathrm{den}}$ using an EMA update rate $\mu\in(0,1)$
(i.e., $\mu$ is the new-sample weight). Concretely, for a stream of observations $x_k$ we use
\[
s \leftarrow (1-\mu)\,s + \mu\,x_k,\qquad s^{(0)}=0.
\]
After $n$ updates, bias correction yields
\[
\widehat{s}_{\mathrm{num}}=\frac{s_{\mathrm{num}}}{1-(1-\mu)^n},\qquad
\widehat{s}_{\mathrm{den}}=\frac{s_{\mathrm{den}}}{1-(1-\mu)^n},\qquad
c^\star=\frac{\widehat{s}_{\mathrm{num}}}{\widehat{s}_{\mathrm{den}}+\delta},
\]
with small $\delta>0$ for numerical stability. After this assignment, $c^\star$ is frozen for the remainder of training.

\section{LayerNorm version of TaperNorm (TaperLN)}\label{app:TaperLN}

This section provides the LayerNorm analogue of the RMSNorm\ formulation used in the main text, including the definition of the layer, the optimal alignment scalar $c^\star$, EMA estimators, and the substitutions required for the theoretical results.

\paragraph{Definition.}
For a token row vector $h\in\mathbb{R}^{1\times d}$, let
\[
\mu_h=\tfrac{1}{d}\sum_{i=1}^d h_i,\qquad
\sigma_h=\sqrt{\tfrac{1}{d}\sum_{i=1}^d (h_i-\mu_h)^2+\varepsilon},
\qquad D_\gamma=\mathrm{diag}(\gamma),\quad D_{\tilde\gamma}=\mathrm{diag}(\tilde\gamma).
\]
Our GPT-2 fine-tuning code uses a TaperLN module that interpolates between the standard LayerNorm branch and a mean-centered scaling branch:
\begin{equation}\label{eq:TaperLN-def}
\mathrm{TaperLN}(h;g)
=
\beta
+
g\,\frac{h-\mu_h}{\sigma_h}\,D_\gamma
+
(1-g)\,c\,(h-\mu_h)\,D_{\tilde\gamma},
\qquad g\in[0,1].
\end{equation}
For $g=1$ the layer equals LayerNorm$(h)$; for $g=0$ it becomes a fixed affine map
built from mean subtraction and a featurewise scaling.

At $g=0$, the map is affine in $h$:
it equals $hA+\beta$ where $A=c\,(I-\tfrac{1}{d}\mathbf{1}\mathbf{1}^\top)D_{\tilde\gamma}$.
Thus, for a following projection $x\mapsto xW+b$ we can fold the affine map into $(W,b)$.

\paragraph{Calibration of $c$ (LayerNorm case).}
At the warmup boundary we align the $g=0$ branch to the $g=1$ LayerNorm branch by choosing $c$ to minimize a one-dimensional least-squares objective analogous to~\eqref{eq:cstar}, but with mean-centering and $\sigma_h$:
let $\bar h \coloneqq h-\mu_h\mathbf{1}$ and define
\begin{equation}\label{eq:cstar-ln}
c^\star
\;\in\;\arg\min_{c\in\mathbb{R}}~\mathbb{E}\!\left[\left\|\frac{\bar h}{\sigma_h}D_\gamma - c\,\bar h D_\gamma\right\|_2^2\right]
\;=\;
\frac{\mathbb{E}\!\left[\|\bar hD_\gamma\|_2^2/\sigma_h\right]}{\mathbb{E}\!\left[\|\bar hD_\gamma\|_2^2\right]}.
\end{equation}
This is the exact analogue of~\eqref{eq:cstar} under the substitutions $h\mapsto \bar h$ and $r(h)\mapsto \sigma_h$.

\paragraph{EMA estimators and freezing.}
During warmup we track the numerator and denominator in~\eqref{eq:cstar-ln} with EMAs and apply the same bias correction as in Appendix~\ref{app:proofs:ema}. At the warmup boundary we set $c\leftarrow c^\star$, copy $\gamma\rightarrow\tilde\gamma$, and freeze $c$ for the remainder of training (while $\gamma$ and $\tilde\gamma$ remain trainable; see Section~\ref{subsec:taper}).

\paragraph{Relationship to the RMSNorm\ variant.}
All statements in Section~\ref{sec:theory} hold for TaperLN under the same regularity assumptions after the substitutions
\[
\frac{h}{r(h)} \;\longrightarrow\; \frac{h-\mu_h}{\sigma_h},
\qquad
\Delta(h)=c\,hD_{\tilde\gamma}-\frac{h}{r(h)}D_\gamma
\;\longrightarrow\;
\Delta_{\mathrm{LN}}(h)=c\,(h-\mu_h)D_{\tilde\gamma}-\frac{h-\mu_h}{\sigma_h}D_\gamma.
\]
In particular, Proposition~\ref{prop:tangent} continues to apply because the ideal LayerNorm map (with $\varepsilon=0$) is $0$-homogeneous for positive scalings. If a bias $\beta$ is included, it does not affect $0$-homogeneity for $\alpha>0$ because it is invariant to radial scalings; Proposition~\ref{prop:tangent} therefore does not require $\beta = 0$.

\section{Training details}\label{app:exp}

This appendix summarizes the TinyStories pre-training setup used in Section~\ref{sec:experiments}.

\paragraph{Compute resources.}
All reported pre-training, fine-tuning, and throughput experiments were carried out on a single NVIDIA H100 GPU.
For TinyStories pre-training, a single run took approximately $4.3$, $4.8$, $8.2$, and $15.4$ hours for widths
$d=64$, $128$, $256$, and $512$, respectively.
These timings use eager PyTorch execution, which allows detailed training diagnostics such as gradient statistics to be recorded.
The throughput microbenchmark hardware and software stack is described in Section~\ref{subsec:efficiency}.

\paragraph{Data and tokenizer.}
We pre-train on TinyStories using a SentencePiece tokenizer trained on TinyStories with vocabulary size $10$k.
Training sequences are formed by sampling random contiguous windows of length $T=512$ from a concatenated token stream.
We train with next-token prediction using a one-token shift.

\paragraph{Model.}
All pre-training models are causal, pre-norm transformers with depth $L=8$.
All widths use $16$ attention heads, RoPE positional encoding, and SwiGLU MLPs with expansion factor $4$.
Linear layers in attention and MLP use no bias.
Output weights are tied to the token embedding matrix.
Dropout is disabled.

\paragraph{Optimization.}
We use AdamW with $\beta=(0.9,0.95)$ and peak learning rate $3\times10^{-4}$.
The schedule uses linear warmup for $5\%$ of total steps followed by cosine decay to zero.
We use no weight decay.
We apply global gradient norm clipping with threshold $1.0$.
The batch size is $16$ sequences of length $512$.

\paragraph{Run lengths.}
Total training steps vary by model width as shown in Table~\ref{tab:ts-steps}.
Warmup uses $5\%$ of total steps for every run.

\begin{table}[t]
\centering
\small
\setlength{\tabcolsep}{10pt}
\renewcommand{\arraystretch}{1.05}
\caption{Total TinyStories pre-training steps by model width.}
\begin{tabular}{@{}rc@{}}
\toprule
Width $d$ & Total steps \\
\midrule
64  & 500{,}000 \\
128 & 500{,}000 \\
256 & 750{,}000 \\
512 & 1{,}000{,}000 \\
\bottomrule
\end{tabular}
\label{tab:ts-steps}
\end{table}

\paragraph{TaperNorm, EMA rates, and scale loss.}
For any layer that is tapered, we keep the gate at $g=1$ during learning-rate warmup.
At the warmup boundary, each tapered layer computes its alignment scalar $c$ using $\gamma$-weighted EMA estimates of
the quantities in Section~\ref{subsec:cstar}, copies $\gamma\rightarrow\tilde\gamma$, and then freezes $c$.
After warmup, we cosine-decay the shared gate value $g(k)$ from $1$ to $0$ over the remaining training steps.

We use the same EMA update rate $\mu=0.01$ (new-sample weight) for both (i) the calibration EMAs used to set $c$
and (ii) the warmup EMA used to set the fixed target $s_{\mathrm{tgt}}$.
When enabled, we apply the fixed-target scale anchoring loss from Section~\ref{subsec:aux} to the pre-logit residual stream
with weight $\lambda=0.1$ for all TinyStories runs. The target $s_{\mathrm{tgt}}$ is set to the bias-corrected EMA of the batch-mean scale during warmup and is held fixed after warmup.

\section{GPT-2 fine-tuning details}\label{app:ft-details}

This appendix summarizes the fine-tuning setup used in Section~\ref{sec:ft-gpt2}.

\paragraph{Relation to the Baroni auxiliary loss.}
\citet{baroni2025transformersdontneedlayernorm} use an auxiliary loss that encourages consistent tokenwise standard deviations across sequence positions,
\begin{equation*}
\mathcal{L}_{\mathrm{Baroni}}
=\lambda\,\mathbb{E}_{b,s}\big(\sigma_{b,s}-\hat\sigma\big)^2,
\qquad
\hat\sigma=\frac{1}{|\mathcal{M}|}\sum_{(b,s)\in \mathcal{M}}\sigma_{b,s},
\end{equation*}
where $\mathcal{M}$ excludes the first token, padding tokens, and end-of-text positions.
This loss has zero penalty whenever $\sigma_{b,s}$ is constant across the selected positions, so it equalizes per-position scale without fixing an absolute pre-logit scale.
This differs from our fixed-target loss, which anchors the pre-logit residual stream to a scalar target set during warmup.

\paragraph{Data and tokenization.}
We fine-tune on the HuggingFace \texttt{openwebtext} dataset.
We tokenize with the corresponding GPT-2 tokenizer, append an end-of-text token after each document, concatenate the token stream,
and split into contiguous blocks of length $1024$.

\paragraph{Optimization and schedule.}
We reuse the fine-tuning hyperparameters of \citet{baroni2025transformersdontneedlayernorm}.
Training uses bf16, global gradient norm clipping with threshold $1.0$, and a cosine learning-rate schedule with warmup.
We use gradient accumulation to match a fixed global token batch size per optimizer step.

\paragraph{Taper schedule, EMA rates, and scale loss.}
Gates stay at $g=1$ until taper start.
At the taper start, each module freezes its alignment scalar $c$ using EMAs of the quantities used for calibration
(see Appendix~\ref{app:TaperLN}), copies $\gamma\rightarrow\tilde\gamma$, and then cosine-decays $g$ to $0$ over the taper window.

We use the same EMA update rate $\mu=0.1$ (new-sample weight) while $g=1$ (prior to taper start) for both (i) the calibration EMAs and (ii) the EMA used to set the fixed target $s_{\mathrm{tgt}}$.
The fixed-target scale loss weight is $\lambda=0.1$ for GPT-2 Small/Medium/Large and $\lambda=0.01$ for GPT-2 XL.

\begin{table}[t]
\centering
\small
\setlength{\tabcolsep}{6pt}
\renewcommand{\arraystretch}{1.05}
\caption{GPT-2 fine-tuning and taper schedule.}
\begin{tabular}{@{}lcccc@{}}
\toprule
Model & total steps & warmup end $k_{\mathrm{w}}$ & taper start $k_{\mathrm{start}}$ & taper end $k_{\mathrm{end}}$ \\
\midrule
GPT-2 Small  & 300 & 25 & 25 & 100 \\
GPT-2 Medium & 500 & 10 & 20 & 200 \\
GPT-2 Large  & 600 & 30 & 30 & 500 \\
GPT-2 XL     & 800 & 20 & 50 & 700 \\
\bottomrule
\end{tabular}
\label{tab:gpt2-schedule}
\end{table}

\paragraph{Evaluation.}
We reuse code published by \citet{baroni2025transformersdontneedlayernorm} for evaluation and report cross-entropy on OpenWebText held-out blocks and on pre-tokenized subsets of The Pile and The Pile-filtered.

\end{document}